\documentclass[conference,twocolumn]{IEEEtran}

\usepackage{times}

\usepackage[numbers]{natbib}
\usepackage[pdftex]{graphicx}
\usepackage{amsmath}
\usepackage{amsfonts,amssymb}
\usepackage{bbm}

\usepackage{amsthm} 
\usepackage{thmtools}
\usepackage{mathtools}
\usepackage{lipsum}  
\usepackage{xspace}
\usepackage{diagbox}

\let\labelindent\relax
\usepackage{enumitem}

\usepackage{caption}
\usepackage{subcaption} 

\usepackage{units}
\usepackage{booktabs} 
\usepackage{tabulary}
\usepackage[usenames,dvipsnames]{color} 
\usepackage{layouts}

\usepackage{algorithm}
\usepackage{algpseudocode}
\algtext*{EndWhile}
\algtext*{EndIf}
\algtext*{EndFor}

\algblock[Input]{Input}{EndInput}
\algtext*{EndInput}
\algblock[Output]{Output}{EndOutput}
\algtext*{EndOutput}
\algblock[Variables]{Variables}{EndVariables}
\algtext*{EndVariables}

\usepackage[utf8]{inputenc}

\date{}

\usepackage{ifthen,version}
\usepackage{color}
\usepackage{soul}
\usepackage{fixltx2e}
\usepackage[pdftex,colorlinks,bookmarks=true]{hyperref} 

\makeatletter
\def\thm@space@setup{\thm@preskip=0pt
\thm@postskip=0pt}
\makeatother

\newtheorem{theorem}{Theorem}
\newtheorem{lemma}{Lemma}
\newtheorem{problem}{Problem}
\newtheorem{definition}{Definition}

\newcommand{\fullFigGap}[0]{\vspace{-1.5\baselineskip}} 

\DeclareMathOperator*{\argmin}{arg\,min}
\DeclareMathOperator*{\argmax}{arg\,max}

\newcommand{\argmaxprob}[1]{\argmax\limits_{#1}}

\newcommand{\abs}[1]{\left|#1 \right|}
\newcommand{\card}[1]{\left|#1\right|}

\DeclareMathOperator*{\suchthat}{\;\; \mbox{s.t.} \;\;}
\newcommand{\expect}[2]{\mathbb{E}_{#1}\left[#2\right]}

\newcommand{\real}[0]{\mathbb{R}}

\newcommand{\bbm}{\begin{bmatrix}}
\newcommand{\ebm}{\end{bmatrix}}

\newcommand{\pair}[2]{\left( #1, #2\right)}

\newcommand{\seq}[2]{\left(#1_{1}, #1_{2}, \ldots, #1_{#2}\right)}

\newcommand{\setst}[2]{\left\lbrace #1\;\;\middle|\;\;#2\right\rbrace}

\newcommand{\overbar}[1]{\mkern 1.5mu\overline{\mkern-1.5mu#1\mkern-1.5mu}\mkern 1.5mu}

\newcommand{\vertexSet}[0]{\mathcal{V}}
\newcommand{\vertex}[0]{v}
\newcommand{\vertexGroup}[0]{V}

\newcommand{\vertexStart}[0]{\vertex_s}
\newcommand{\Path}[0]{\xi}
\newcommand{\PathSet}[0]{\Xi}

\newcommand{\world}[0]{\phi}
\newcommand{\worldTwo}[0]{\world'}
\newcommand{\worldSet}[0]{\mathcal{M}}
\newcommand{\meas}[0]{y}
\newcommand{\measSet}[0]{\mathcal{Y}}
\newcommand{\measFnDef}[0]{\mathcal{H}}
\newcommand{\measFn}[2]{\measFnDef\left(#1, #2\right)}

\newcommand{\utilityFnDef}[0]{\mathcal{F}}
\newcommand{\utilityFn}[2]{\utilityFnDef\left(#1, #2\right)}
\newcommand{\marginalGain}[2]{\Delta_\utilityFnDef\left(#1, #2\right)}

\newcommand{\costFnDef}[0]{\mathcal{T}}
\newcommand{\costFn}[2]{\costFnDef\left(#1, #2\right)}
\newcommand{\costBudget}[0]{B}

\newcommand{\adaptivePath}[0]{\sigma}

\newcommand{\adaptivePathGreedy}[0]{\adaptivePath_{\mathrm{greedy}}}

\newcommand{\state}[0]{s}
\newcommand{\stateSet}[0]{\mathcal{S}}
\newcommand{\action}[0]{a}

\newcommand{\actionSet}[0]{\mathcal{A}}
\newcommand{\actionSetFeas}[2]{\actionSet_{\mathrm{feas}}\left(#1, #2\right)}
\newcommand{\transFnDef}[0]{\Omega}
\newcommand{\transFn}[3]{\transFnDef{}\left(#1, #2, #3\right)}
\newcommand{\rewardFnDef}[0]{R}
\newcommand{\rewardFn}[3]{\rewardFnDef{}\left(#1, #2, #3\right)}
\newcommand{\rewardFnAgg}[3]{C\left(#1, #2, #3\right)}

\newcommand{\policy}[0]{\pi}
\newcommand{\policySet}[0]{\Pi}

\newcommand{\valueFn}[2]{V^{#1}_{#2}}
\newcommand{\valueFnAgg}[2]{\overbar{V}^{#1}_{#2}}

\newcommand{\QFn}[2]{Q^{#1}_{#2}}
\newcommand{\QFnAgg}[2]{\overbar{Q}^{#1}_{#2}}
\newcommand{\valuePol}[1]{J\left(#1\right)}
\newcommand{\valuePolAgg}[1]{\overbar{J}\left(#1\right)}

\newcommand{\obs}[0]{o}
\newcommand{\obsSet}[0]{\mathcal{O}}
\newcommand{\obsFnDef}[0]{Z}
\newcommand{\obsFn}[4]{\obsFnDef{}\left(#1, #2, #3, #4\right)}

\newcommand{\belief}[0]{\psi}

\newcommand{\policyBel}[0]{\tilde{\pi}}
\newcommand{\policySetBel}[0]{\tilde{\Pi}}
\newcommand{\valueFnBel}[2]{\tilde{V}^{#1}_{#2}}

\newcommand{\QFnBel}[2]{\tilde{Q}^{#1}_{#2}}

\newcommand{\dataset}[0]{\mathcal{D}}
\newcommand{\policyLEARN}[0]{\hat{\pi}}
\newcommand{\numLearnIter}[0]{N}
\newcommand{\mixfrac}[1]{\beta_{#1}}
\newcommand{\numDatapoints}[0]{m}

\newcommand{\policyOR}[0]{\pi_{\mathrm{OR}}}
\newcommand{\policyMix}[0]{\pi_{\mathrm{mix}}}

\newcommand{\policyMDP}[0]{\pi_{\mathrm{MDP}}}
\newcommand{\policyORBel}[0]{\tilde{\pi}_{\mathrm{OR}}}

\newcommand{\QVal}[0]{Q}

\newcommand{\lossFnPolicyDef}[0]{\mathcal{L}}
\newcommand{\lossFnPolicy}[3]{\lossFnPolicyDef{}\left( #1, #2, #3\right)}

\newcommand{\errclass}[0]{\varepsilon_{\mathrm{class}}}
\newcommand{\errclassAgg}[0]{\overbar{\varepsilon}_{\mathrm{class}}}

\newcommand{\errreg}[0]{\varepsilon_{\mathrm{reg}}}
\newcommand{\errregAgg}[0]{\overbar{\varepsilon}_{\mathrm{reg}}}

\newcommand{\lossi}[0]{l_i}

\newcommand{\lossiAgg}[0]{\overbar{l_i}}

\newcommand{\feature}[0]{f}
\newcommand{\featureIG}[0]{\feature{}_{\mathrm{IG}}}
\newcommand{\featureMot}[0]{\feature{}_{\mathrm{mot}}}

\newcommand{\etal}[0]{et al.\xspace}
\newcommand{\aggrevate}[0]{\textsc{AggreVaTe}\xspace}
\newcommand{\Dagger}[0]{\textsc{DAgger}\xspace}
\newcommand{\FT}[0]{\textsc{ForwardTraining}\xspace}
\newcommand{\RearSideVoxel}[0]{Rear Side Voxel\xspace}
\newcommand{\AverageEntropy}[0]{Average Entropy\xspace}

\newcommand{\knownunc}[0]{\textsc{Known-Unc}\xspace}
\newcommand{\knowncon}[0]{\textsc{Known-Con}\xspace}
\newcommand{\hiddenunc}[0]{\textsc{Hidden-Unc}\xspace}
\newcommand{\hiddencon}[0]{\textsc{Hidden-Con}\xspace}
\newcommand{\algRewFT}[0]{\textsc{RewardFT}\xspace}
\newcommand{\algQvalFT}[0]{\textsc{QvalFT}\xspace}
\newcommand{\algRewAgg}[0]{\textsc{RewardAgg}\xspace}
\newcommand{\algQvalAgg}[0]{\textsc{QvalAgg}\xspace}

\graphicspath{{figs/}}

\title{\LARGE \bf
Adaptive Information Gathering via Imitation Learning
}

\author{\authorblockN{Sanjiban Choudhury\authorrefmark{1},
Ashish Kapoor\authorrefmark{2},
Gireeja Ranade\authorrefmark{2}, 
Sebastian Scherer\authorrefmark{1} and
Debadeepta Dey\authorrefmark{2}}
\authorblockA{\authorrefmark{1}The Robotics Institute\\
Carnegie Mellon University,
Pittsburgh, PA 15232\\ Email: \{sanjibac,basti\}@andrew.cmu.edu}
\authorblockA{\authorrefmark{2}Microsoft Research, Redmond USA\\
Email: \{akapoor,giranade,dedey\}@microsoft.com}}

\begin{document}

\maketitle


\begin{abstract}
In the adaptive information gathering problem, a policy is required to select an informative sensing location using the history of measurements acquired thus far. 
While there is an extensive amount of prior work investigating effective practical approximations using variants of Shannon's entropy, the efficacy of such policies heavily depends on the geometric distribution of objects in the world. 
On the other hand, the principled approach of employing online POMDP solvers is rendered impractical by the need to \emph{explicitly} sample online from a posterior distribution of world maps. 

We present a novel data-driven imitation learning framework to efficiently train information gathering policies. 
The policy imitates a \emph{clairvoyant oracle} - an oracle that at train time has full knowledge about the world map and can compute maximally informative sensing locations. 
We analyze the learnt policy by showing that offline imitation of a clairvoyant oracle is \emph{implicitly} equivalent to online oracle execution in conjunction with posterior sampling. 
This observation allows us to obtain powerful near-optimality guarantees for information gathering problems possessing an adaptive sub-modularity property. 
As demonstrated on a spectrum of 2D and 3D exploration problems, the trained policies enjoy the best of both worlds - they adapt to different world map distributions while being computationally inexpensive to evaluate.
\end{abstract}

\section{Introduction}

This paper \footnote{This work was conducted by Sanjiban Choudhury as part of a summer internship at Microsoft Research, Redmond, USA.} examines the following information gathering problem - given a hidden world map, sampled from a prior distribution, the goal is to successively visit sensing locations such that the amount of relevant information uncovered is maximized while not exceeding a specified fuel budget. This problem fundamentally recurs in mobile robot applications such as autonomous mapping of environments using ground and aerial robots \cite{Charrow-RSS-15,heng2015efficient}, monitoring of water bodies \cite{hollinger2013sampling} 
and inspecting models for 3D reconstruction \cite{isler2016information,hollinger2011active}.

The nature of ``interesting'' objects in an environment and their spatial distribution influence the optimal trajectory a robot might take to explore the environment. As a result, it is important that a robot learns about the type of environment it is exploring as it acquires more information and adapts its exploration trajectories accordingly. This adaptation must be done online, and we provide such algorithms in this paper.

Consider a robot equipped with a sensor (RGBD camera) that needs to generate a map of an unknown environment. It is given a prior distribution about the geometry of the world - such as a distribution over narrow corridors, or bridges or power-lines. At every time step, the robot visits a sensing location and receives a sensor measurement (e.g. depth image) that has some amount of information utility (e.g. surface coverage of objects with point cloud). If the robot employs a non-adaptive lawnmower-coverage pattern, it can potentially waste time and fuel visiting ineffective locations. On the other hand, if it uses the history of measurements to infer the geometry of unexplored space, it can plan efficient information-rich trajectories. This incentivizes training policies to optimally gather information on the distribution of worlds the robot expects to encounter.

Even though its is natural to think of this problem setting as a POMDP, we frame this problem as a novel data-driven imitation learning problem \cite{ross2010efficient}. We propose a framework that trains a policy on a dataset of worlds by imitating a \emph{clairvoyant oracle}. During the training process, the oracle has full knowledge about the world map (hence clairvoyant) and visits sensing locations that would maximize information. The policy is then trained to imitate these movements as best as it can using partial knowledge from the current history of movements and measurements. As a result of our novel formulation, we are able to sidestep a number of challenging issues in POMDPs like explicitly computing posterior distribution over worlds and planning in belief space. 

Our contributions are as follows:
\begin{enumerate}
\item We map the information gathering problem to a POMDP and present an approach to solve it using imitation learning of a clairvoyant oracle. 
\item We present a framework to train such a policy on the non i.i.d distribution of states induced by the policy itself. 
\item We analyze the learnt policy by showing that offline imitation of a clairvoyant oracle is equivalent to online oracle execution in conjunction with posterior sampling. 
\item We present results on a variety of datasets to demonstrate the efficacy of the approach.
\end{enumerate}

The remainder of this paper is organized as follows. Section \ref{sec:prob} presents the formal problem and Section \ref{sec:pomdp_imitate} maps it to the imitation learning framework. The algorithms and analysis is presented in Section \ref{sec:apprach}. Section \ref{sec:res} presents experimental results with conclusions in Section \ref{sec:conc}.


\section{Background}
\label{sec:prob}
\subsection{Notation}

Let $\vertexSet$ be a set of nodes corresponding to all sensing locations.
The robot starts at node $\vertexStart$.
Let $\Path = \seq{\vertex}{p}$ be a sequence of nodes (a path) such that $\vertex_1 = \vertexStart$. Let $\PathSet$ be the set of all such paths.
Let $\world \in \worldSet$ be the world map. 
Let $\meas \in \measSet$ be a measurement received by the robot. Let $\measFnDef{}: \vertexSet \times \worldSet \to \measSet$ be a measurement function. When the robot is at node $\vertex$ in a world map $\world$, the measurement $\meas$ received by the robot is $\meas = \measFn{\vertex}{\world}$. 
Let $\utilityFnDef{}: 2^\vertexSet \times \worldSet \to \real_{\geq 0}$ be a utility function. For a path $\Path$ and a world map $\world$, $\utilityFn{\Path}{\world}$ assigns a utility to executing the path on the world. 
Given a node $\vertex \in \vertexSet$, a set of nodes $\vertexGroup \subseteq \vertexSet$ and world $\world$, the discrete derivative of the utility function $\utilityFnDef$ is $\marginalGain{\vertex \mid \vertexGroup}{\world} = \utilityFn{\vertexGroup \cup \{ \vertex \}}{\world} - \utilityFn{\vertexGroup}{\world}$
Let $\costFnDef{}: \PathSet \times \worldSet \to \real_{\geq 0}$ be a travel cost function. For a path $\Path$ and a world map $\world$, $\costFn{\Path}{\world}$ assigns a travel cost for executing the path on the world. Fig.~\ref{fig:problem} shows an illustration.

\subsection{Problems with Known World Maps}
\label{sec:prob:known}

\begin{figure}[t!]
    \centering
    \includegraphics[width=\columnwidth]{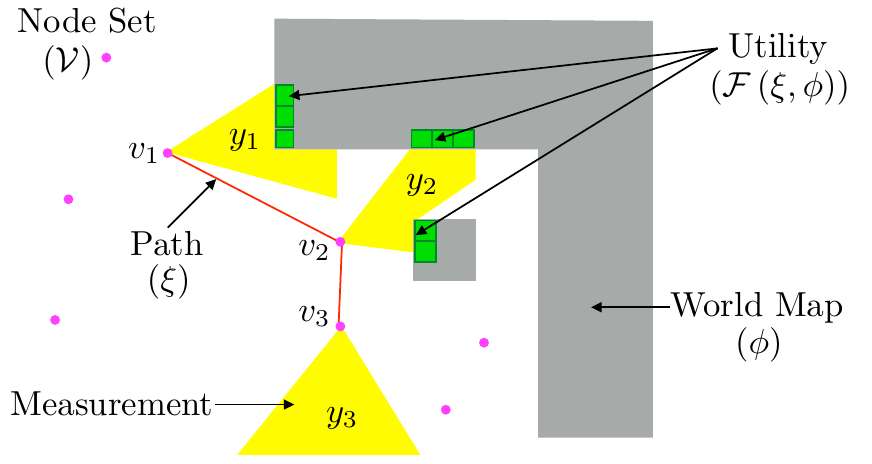}
    \caption{%
    The adaptive information gathering problem. Given a world map $\world$, the robot plans a path $\Path$ which visits a node $\vertex_i \in \vertexSet$ and receives measurement $y_i$, such that information gathered (utility)  $\utilityFn{\Path}{\world}$ is maximized.
    \label{fig:problem}
}
\end{figure}

We define four variants of the information gathering problem. For the first two variants, the world map $\world$ is known and can be evaluated while computing a path $\Path$.

\begin{problem}[\knownunc: Known World Map; Unconstrained Travel Cost] \label{prob:known:unc}
Given a world map $\world$ and a time horizon $T$, find a path $\Path$ that maximizes utility
\begin{equation}
\begin{aligned}
\argmaxprob{\Path \in \PathSet} \quad & \utilityFn{\Path}{\world} \\
\suchthat                             & \card{\Path} \leq T+1
\end{aligned}
\end{equation}
\end{problem}

\begin{problem}[\knowncon: Known World Map; Constrained Travel Cost] \label{prob:known:cons}
Problem~\ref{prob:known:unc} with a travel cost budget $\costBudget$ 
\begin{equation}
\begin{aligned}
\argmaxprob{\Path \in \PathSet} \quad & \utilityFn{\Path}{\world} \\
\suchthat                             & \costFn{\Path}{\world} \leq \costBudget \\
                                      &  \card{\Path} \leq T+1
\end{aligned}
\end{equation}
\end{problem}

Problem~\ref{prob:known:unc} is a set function maximization problem which in general can be NP-Hard (\citet{krause2012submodular}). However, the utility function $\utilityFnDef$ is a \emph{monotone submodular} function. For such functions, it has been shown that greedy strategies achieve near-optimality (\citet{krause2008efficient,Krause:2007:NOS:1619797.1619913}).

Problem~\ref{prob:known:cons} introduces a routing constraint (due to $\costFnDef$) for which greedy approaches can perform arbitrarily poorly. \citet{chekuri2005recursive,singh2007efficient} propose a quasi-polynomial time recursive greedy approach to solving this problem. \citet{iyer2013submodular} solve a related problem (submodular knapsack constraints) using an iterative greedy approach which is generalized by \citet{zhang2016submodular}. \citet{yu2014correlated} propose a mixed integer approach to solve a related correlated orienteering problem. \citet{hollinger2013sampling} propose a sampling based approach.

\subsection{Problems with Hidden World Maps}
\label{sec:prob:hidden}

We now consider the setting where the world map $\world$ is hidden. Given a prior distribution $P(\world)$, it can be inferred only via the measurements $\meas_i$ received as the robot visits nodes $\vertex_i$. Hence, instead of solving for a fixed path, we compute a policy that maps history of measurements received and nodes visited to decide which node to visit. 

\begin{problem}[\hiddenunc: Hidden World Map; Unconstrained Travel Cost] \label{prob:hidden:unc}
Given a distribution of world maps, $P(\world)$, a time horizon $T$, find a policy that at time $t$, maps the history of nodes visited $\{ \vertex_i \}_{i=1}^{t-1}$ and measurements received $\{ \meas_i \}_{i=1}^{t-1}$ to compute node $\vertex_t$ to visit at time $t$, such that the expected utility is maximized. 
\end{problem}

\begin{problem}[\hiddencon: Hidden World Map; Constrained Travel Cost] \label{prob:hidden:cons}
Problem~\ref{prob:hidden:unc} with a travel cost budget $\costBudget$
\end{problem}

Due to the hidden world map $\world$, it is not straight forward to apply the approaches discussed in Section~\ref{sec:prob:known} - methods have to reason about how $P(\world \; | \; \{ \vertex_i \}_{i=1}^{t-1} , \{ \meas_i \}_{i=1}^{t-1})$ will evolve. However, the utility function $\utilityFnDef$  has an additional property of \emph{adaptive submodularity} \cite{golovin2011adaptive}. Hence, applying greedy strategies to Problem~\ref{prob:hidden:unc} has near-optimality guarantees (\citet{golovin2010near}, Javdani \etal \cite{Javdani_2013_7419,Javdani_2014_7555}, Chen \etal \cite{AAAI159841,DBLP:journals/corr/ChenHK16a} ).

Problem~\ref{prob:hidden:cons} does not enjoy the adaptive submodularity property. Hollinger \etal \cite{hollinger2012active,hollinger2011active} propose a heuristic based approach to select a subset of informative nodes and perform minimum cost tours. Singh \etal \cite{Singh:2009:NAI:1661445.1661741} replan every step using a non-adaptive information path planning algorithm. Inspired by adaptive TSP approaches by Gupta \etal \cite{gupta2010approximation}, Lim \etal  \cite{lim2016adaptive,NIPS2015_6005} propose recursive coverage algorithms to learn policy trees. However such methods cannot scale well to large state and observation spaces. \citet{heng2015efficient} make a modular approximation of the objective function. \citet{isler2016information} survey a broad number of myopic information gain based heuristics that work well in practice but have no formal guarantees.


\section{POMDPs and Imitation Learning}
\label{sec:pomdp_imitate}

\subsection{Mapping Problems to a POMDP}

We now map Problems \hiddenunc and \hiddencon to a Partially Observable Markov Decision Process (POMDP). The POMDP is a tuple $(\stateSet, \worldSet, \actionSet, \transFnDef, \rewardFnDef, \obsSet, \obsFnDef, T)$ defined upto a fixed finite horizon $T$. It is defined over an augmented state space comprising of the ego-motion state space $\stateSet$ (which we will refer to as simply  the state space) and the space of world maps $\worldSet$. The first component, $\stateSet$, is fully observable while the second component, $\worldSet$, is partially observable through observations received.

Let the state, $\state_t \in \stateSet$, be the set of nodes visited, $\state_t = \seq{\vertex}{t}$. 
Let the action, $\action_t \in \actionSet$ be the node visited $\action_t = \vertex_{t+1}$. 
Given a world map $\world$, at state $\state$, the utility of $\action$ is $\utilityFn{\state \cup \action}{\world}$.
For Problem \hiddencon, let $\actionSetFeas{\state}{\world} \subset \actionSet$ be the set of feasible actions defined as
\begin{equation}
   \actionSetFeas{\state}{\world} = \setst{\action}{\action \in \actionSet, \costFn{\state \cup \action}{\world} \leq \costBudget}
\end{equation} 
The state transition function, $\transFn{\state}{\action}{\state'} = P(\state' | \state, \action)$, is the deterministic function $\state' = \state \cup \action$. 
The one-step-reward function, $\rewardFn{\state}{\world}{\action} \in [0,1]$, is defined as the normalized marginal gain of the utility function, $\rewardFn{\state}{\world}{\action} = \frac{\marginalGain{\action \mid \state}{\world}}{\utilityFn{\actionSet}{\world}}$.
Let the observation, $\obs_t \in \obsSet$ be the measurement $\obs_t = \meas_t$. The observation model, $\obsFn{\state}{\action}{\world}{\obs} = P(\obs | \state, \action, \world)$ is the deterministic function $\obs = \measFn{\state \cup \action}{\world}$.

We define the belief, $\belief_t$, to be the history of state, action, observation tuple received so far, i.e. $\{(\state_i, \action_i, \obs_i)\}_{i=1}^{t}$. The belief transition function $P(\belief' | \belief, \action)$ can be computed from $\transFnDef$ and $\obsFnDef$.
Let $\policyBel(\state, \belief) \in \policySetBel$ be a policy that maps state $\state$ and belief $\belief$ to a feasible action $\action \in \actionSetFeas{\state}{\world}$. 
The value of executing a policy $\policyBel$ for $t$ time steps starting at state $\state$ and belief $\belief$ is the expected cumulative reward obtained by $\policyBel$:
\begin{equation}
\valueFnBel{\policyBel}{t}(\state, \belief) = \sum\limits_{i=1}^{t} \expect{
\substack{\belief_i \sim P(\belief' \mid \belief, \policyBel, i), \\
\world \sim P(\world \mid \belief_i) \\ \state_i \sim P(\state' \mid \state, \policyBel, i)}
 }{\rewardFn{\state_i}{\world}{\policyBel(\state_i,\belief_i)}}
\end{equation}
where $P(\belief' | \belief, \policyBel, i)$ is the distribution of beliefs at time $i$ starting from $\belief$ and following policy $\policyBel$. Similarly $P(\state' | \state, \policyBel, i)$ is the distribution of states.
$P(\world | \belief_i)$ is the posterior distribution on worlds given the belief $\belief_i$. 

The state-action value function $\QFnBel{\policyBel}{t}$ is defined as the expected sum of one-step-reward and value-to-go
\begin{equation}
\begin{aligned}
\label{eq:pomdp:q}
\QFnBel{\policyBel}{t}(\state, \belief, \action) =& \expect{\world \sim P(\world \mid \belief)}{\rewardFn{\state}{\world}{\action}} + \\
                        &\expect{\belief' \sim P(\belief' \mid \belief, \action), \state' \sim P(\state' \mid \state, \action)}{\valueFnBel{\policyBel}{t-1}(\state', \belief')}
\end{aligned}
\end{equation}

The optimal POMDP policy is obtained by minimization of the expected state-action value function
\begin{equation}
\label{eq:pomdp:opt_policy}
\policy^* = \argmax\limits_{\policyBel \in \policySetBel} 
\expect{
\substack{t\sim U(1:T), \\
 \state \sim P(\state \mid \policyBel, t), \\
 \belief \sim P(\belief \mid \policyBel, t)}
 }{\QFnBel{\policyBel}{T-t+1}(\state, \belief, \policyBel(\state,\belief))}
\end{equation}
where $U(1:T)$ is a uniform distribution over the discrete interval $\{1, \dots, T\}$, $P(\state \mid \policyBel, t)$ and $P(\belief \mid \policyBel, t)$ are the  posterior distribution over states and beliefs following policy $\policyBel$ for $t$ steps.
The value of a policy $\policyBel \in \policySetBel$ for $T$ steps on a distribution of worlds $P(\world)$, starting states $P(\state)$ and starting belief $P(\belief)$ is the expected value function for the full horizon
\begin{equation}
\valuePol{\policyBel} = \expect{\state \sim P(\state), \belief \sim P(\belief)}{\valueFn{\policyBel}{T}(\state, \belief)}
\end{equation}

Just as Problems \hiddenunc and \hiddencon map to a POMDP, Problems \knownunc and \knowncon map to the corresponding MDP. While we omit the details for brevity, the MDP defines a corresponding policy $\policy(\state, \world) \in \policySet$, value function $\valueFn{\policy}{t}(\state, \world)$, state-action value function $\QFn{\policy}{t}(\state, \world, \action)$ and optimal policy $\policyMDP$.

Online POMDP planning also has a large body of work (see \citet{ross2008online}). Although there exists fast solvers such as POMCP (Silver and Veness \cite{silver2010monte}) and DESPOT (Somani \etal \cite{somani2013despot}), the space of world maps is too large for online planning. An alternative class of approaches is model-free policy improvement (\cite{peters2006policy}). While these methods make very little assumptions about the problem, they are local and require careful initialization.

\subsection{Imitation Learning}
\label{sec:pomdp_imitate:imitation_learning}
An alternative to policy improvement approaches is to train policies to imitate reference policies (or oracle policies). This is a suitable choice for scenarios where the problem objective is to imitate a user-defined policy. This is also a useful approach in scenarios where there exist good oracle policies for the original problem, however these policies cannot be executed online (e.g due to computational complexity) hence requiring imitation via an offline training phase.

We now formally define imitation learning as applied to our setting. Given a policy $\policyBel$, we define the distribution of states $P(\state | \policyBel)$ and beliefs $P(\belief | \policyBel)$ induced by it (termed as \emph{roll-in}). Let $\lossFnPolicy{\state}{\belief}{\policyBel}$ be a loss function that captures how well policy $\policyBel$ imitates an oracle. Our goal is to find a policy $\policyLEARN$ which minimizes the expected loss as follows.

\begin{equation}
\label{eq:imitation_learning}
\policyLEARN = \argmin\limits_{\policyBel \in \policySetBel} \expect{ 
\state \sim P(\state \mid \policyBel), 
\belief \sim P(\belief \mid \policyBel)}
{\lossFnPolicy{\state}{\belief}{\policyBel}}
\end{equation} 

This is a non-i.i.d supervised learning problem. Ross and Bagnell \cite{ross2010efficient} propose \FT to train a non-stationary policy (one policy $\policyLEARN_t$ for each timestep), where each policy $\policyLEARN_t$ can be trained on distributions induced by previous policies ($\policyLEARN_1, \dots, \policyLEARN_{t-1}$). While this has guarantees, it is impractical given a different policy is needed for each timestep. For training a single policy, Ross and Bagnell \cite{ross2010efficient} show how such problems can be reduced to no-regret online learning using dataset aggregation (\Dagger). The loss function they consider $\lossFnPolicyDef$ is a mis-classification loss with respect to what the expert demonstrated. Ross and Bagnell \cite{ross2014reinforcement} extend the approach to the reinforcement learning setting where $\lossFnPolicyDef$ is the reward-to-go of an oracle reference policy by aggregating \emph{values} to imitate (\aggrevate).

\subsection{Solving POMDP via Imitation of a Clairvoyant Oracle}

To examine the suitability of imitation learning in the POMDP framework, we compare the training rules (\ref{eq:pomdp:opt_policy}) and (\ref{eq:imitation_learning}). We see that a good candidate loss function $\lossFnPolicy{\state}{\belief}{\policyBel}$ should incentivize maximization of $\QFnBel{\policyBel}{T-t+1}(\state, \belief, \policyBel(\state,\belief))$. A suitable approximation of the optimal value function $\QFnBel{\policy^*}{T-t+1}$ that can be computed at train time would suffice. In this work we define cumulative reward gathered by a \emph{clairvoyant oracle} as the value function to imitate\footnote{Imitation of a clairvoyant oracle in information gathering has been explored by \citet{choudhury2016learning}. We subsume the presented algorithm, \textsc{ExpLOre}, in our framework (as Algorithm \algQvalAgg) and instead focus on the theoretical insight on what it means to imitate a clairvoyant oracle. \citet{kahn2016plato} has explored a similar idea in the context of reactive obstacle avoidance where a clairvoyant MPC is imitated. This too can be subsumed in our framework (as Algorithm \algRewAgg). Additionally, we provide analysis highlighting when such an approach is effective (in Section~\ref{sec:pomdp_imitate:hallucinating})}.

\begin{definition}[Clairvoyant Oracle]
Given a distribution of world map $P(\world)$, a clairvoyant oracle $\policyOR(\state, \world)$ is a policy that maps state $\state$ and world map $\world$ to a feasible action $\action \in \actionSetFeas{\state}{\world}$ such that it approximates the optimal MDP policy, $\policyOR \approx \policyMDP = \argmaxprob{\policy \in \policySet}\valuePol{\policy}$.
\end{definition}

The term \emph{clairvoyant} is used because the oracle has full access to the world map $\world$ at train time. The oracle can be used to compute state-action value as follows

\begin{equation}
\label{eq:qvaloracle}
\QFn{\policyOR}{t}(\state, \world, \action) = \rewardFn{\state}{\world}{\action} + \expect{\state' \sim P(\state' \mid \state, \action)}{\valueFn{\policyOR}{t-1}(\state', \world)} 
\end{equation}

Our approach is to imitate the oracle during training. This implies that we train a policy $\policyLEARN$ by solving the following optimization problem

\begin{equation}
\label{eq:imitateClairvoyantOracle}
\policyLEARN = \argmax\limits_{\policyBel \in \policySetBel} \expect{
\substack{t\sim U(1:T), \\
\state \sim P(\state \mid \policyBel, t), \\
\world \sim P(\world), \\
\belief \sim P(\belief \mid \world, \policyBel, t)}}
{\QFn{\policyOR}{T-t+1}(\state, \world, \policyBel(\state,\belief))}
\end{equation}

While we will define training procedures to concretely realize (\ref{eq:imitateClairvoyantOracle}) later in Section~\ref{sec:apprach}, we offer some intuition behind this approach. Since the oracle $\policyOR$ knows the world map $\world$, it has appropriate information to assign a value to an action $a$. The policy $\policyLEARN$ attempts to imitate this action from the partial information content present in its belief $\belief$. Due to this realization error, the policy $\policyLEARN$ visits a different state, updates its belief, gains more information, and queries the oracle for the best action. Hence while the learnt policy can make mistakes in the beginning of an episode, with time it gets better at imitating the oracle. 
\subsection{Analysis using a Hallucinating Oracle}
\label{sec:pomdp_imitate:hallucinating}
The learnt policy imitates a clairvoyant oracle that has access to more information (world map $\world$ compared to belief $\belief$). Hence, the realizability error of the policy is due to two terms - firstly the information mismatch and secondly the expressiveness of feature space. This realizability error can be hard to bound making it difficult to bound the performance of the learnt policy. This motivates us to introduce a hypothetical construct, a \emph{hallucinating oracle}, to alleviate the information mismatch.

\begin{definition}[Hallucinating Oracle]
Given a prior distribution of world map $P(\world)$ and roll-outs by a policy $\policyBel$, a hallucinating oracle $\policyORBel$ computes the instantaneous posterior distribution over world maps and takes the action with the highest expected state-action value as computed by the clairvoyant oracle. 
\begin{equation}
  \policyORBel = \argmax\limits_{\action \in \actionSet} \expect{\worldTwo \sim P(\worldTwo \mid \belief, \policyBel, t)}{\QFn{\policyOR}{T-t+1}(\state, \worldTwo, \action)}
\end{equation}
\end{definition}

We will now show that by imitating a clairvoyant oracle we effectively imitate the corresponding hallucinating oracle 

\begin{lemma}
The \textbf{offline} imitation of \textbf{clairvoyant} oracle (\ref{eq:imitateClairvoyantOracle}) is equivalent to sampling \textbf{online} a world from the posterior distribution and executing a \textbf{hallucinating} oracle as shown\footnote{Refer to supplementary for all proofs}

\begin{equation*}
\policyLEARN = \argmax\limits_{\policyBel \in \policySetBel} \expect{
\substack{t\sim U(1:T), \\
\state \sim P(\state \mid \policyBel, t), \\
\world \sim P(\world), \\
\belief \sim P(\belief \mid \world, \policyBel, t)}}
{\QFn{\policyORBel}{T-t+1}(\state, \world, \policyBel(\state,\belief))}
\end{equation*}

\end{lemma}

Note that a hallucinating oracle uses the same information content as the learnt policy. Hence the realization error is purely due to the expressiveness of the feature space. However, we now have to analyze the performance of the hallucinating oracle which chooses the best action at a time step given its current belief. We will see in Section \ref{sec:approach:one_step_reward} that this policy is near-optimal for Problem \hiddenunc. For Problem \hiddencon, while this does not have any such guarantees, there is evidence to show this is an effective policy as alluded to in \citet{Koval-RSS-14}.

\section{Approach}
\label{sec:apprach}

\begin{figure*}[!htp]
    \centering
    \includegraphics[width=\textwidth]{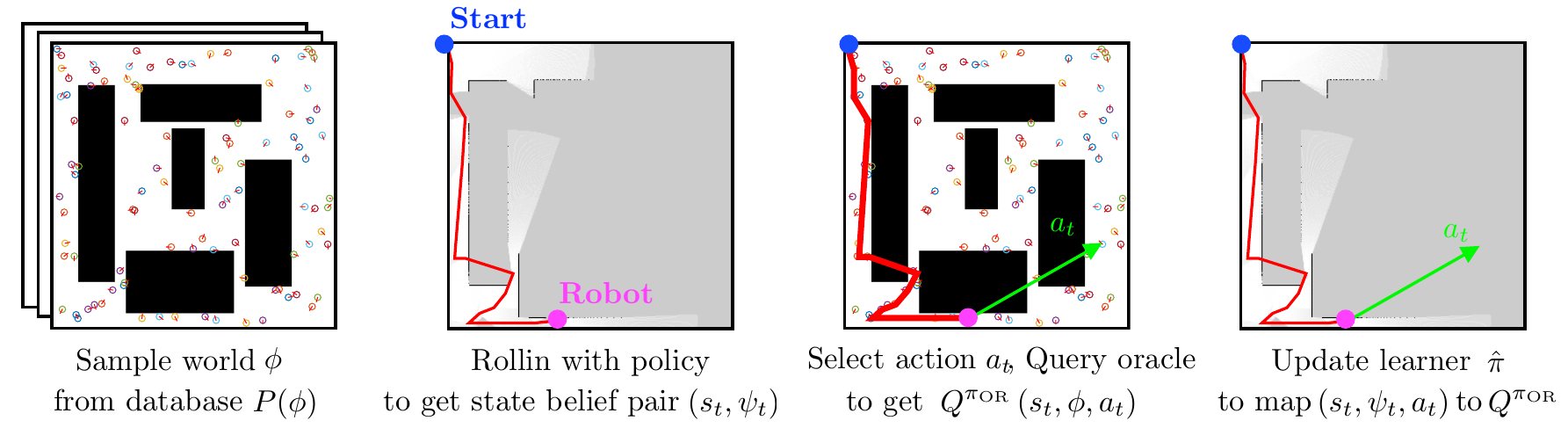}
    \caption{%
    An abstract overview of the imitation learning architecture where a learner $\policyLEARN$ is trained to imitate a clairvoyant oracle $\policyOR$. There are 4 key steps. Step 1: A world map $\world$ is sampled from database representing $P(\world)$. Step 2: A policy is used to roll-in on $\world$ to a timestep $t$ to get state $\state_t$ and belief $\belief_t$. Step 3: A random action $a_t$ is chosen and a clairvoyant oracle $\policyOR$ is given full access to world map $\world$ to compute the cumulative reward to go $\QFn{\policyOR}{T-t+1}\left(\state_t, \world, \action_t \right)$ Step 4: The learnt policy  $\policyLEARN$ is updated to map $\left(\state_t, \belief_t, \action_t \right)$ to $\QVal^{\policyOR}$.
        \fullFigGap}
    \label{fig:alg:overview}
\end{figure*}%

\subsection{Overview}

We introduced imitation learning and its applicability to POMDPs in Section~\ref{sec:pomdp_imitate}. We now present a set of algorithms to concretely realize the process. The overall idea is as follows - we are training a policy $\policyLEARN(\state, \belief)$ that maps features extracted from state $\state$ and belief $\belief$ to an action $\action$. The training objective is to imitate a clairvoyant oracle that has access to the corresponding world map $\world$. Fig.~\ref{fig:alg:overview} shows an abstract overview. In order to define concrete algorithms, there are two degrees of freedom that need to be specified

\begin{enumerate}
  \item \textit{Clairvoyant oracle to imitate}: The choice of the clairvoyant oracle defines $\QFn{\policyOR}{t}(\state, \world, \action)$ in (\ref{eq:qvaloracle}). Depending on whether we are solving Problem \hiddenunc or \hiddencon, we explore two different kinds of oracles
  \begin{enumerate}[label=(\alph*)]
    \item \textit{Clairvoyant one-step-reward}: For Problem \hiddenunc, we use the one-step-reward $\rewardFn{\state}{\world}{\action}$ in place of $\QFn{\policyOR}{t}(\state, \world, \action)$. This corresponds to greedily imitating the immediate reward. We will see in Section~\ref{sec:approach:one_step_reward} that this has powerful near-optimality guarantees.
    \item \textit{Clairvoyant reward-to-go}: For Problem \hiddencon, we define an oracle policy $\policyOR$ that approximately solves the underlying MDP and use its cumulative reward-to-go to compute $\QFn{\policyOR}{t}(\state, \world, \action)$ (refer (\ref{eq:qvaloracle})).
  \end{enumerate}
  \item \textit{Stationary / Non-stationary policy}: As mentioned in Section.~\ref{sec:pomdp_imitate:imitation_learning}, imitation learning deals with non i.i.d distributions induced by the policy themselves. Depending on the choice of policy type (non-stationary / stationary), a corresponding training algorithm exists that offers guarantees
  \begin{enumerate}[label=(\alph*)]
    \item \textit{Non-stationary policy}: For the non-stationary case, we have a policy for each timestep $\policyLEARN^1, \ldots, \policyLEARN^{T}$. While this can be trained using the \FT algorithm \cite{ross2010efficient}, there are several drawbacks. Firstly, it is impractical to have a different policy for each time-step as it scales with $T$. Secondly, the training has to proceed sequentially. Thirdly, each policy operates on data for only that time-step, thus preventing generalizations across timesteps. However, the procedure is presented for completeness.
    \item \textit{Stationary policy}: A single stationary policy $\policyLEARN$ can be trained using the \aggrevate algorithm \cite{ross2014reinforcement} that reduces the problem to no-regret online learning setting. The training procedure is an interactive process where data is aggregated to refine the policy. The advantages are that the policy uses data across all timesteps, only one policy needs to be tracked and the training process can be stopped arbitrarily.
  \end{enumerate}
\end{enumerate}

\subsection{Algorithm}

We now present concrete algorithms to realize the training procedure. Given the two axes of variation - problem and policy type - we have four possible algorithms 
\begin{enumerate}
    \item \algRewFT: Imitate one-step-reward using non-stationary policy by \FT (Alg.~\ref{alg:FT})
    \item \algQvalFT: Imitate reward-to-go using non-stationary policy by \FT (Alg.~\ref{alg:FT})
    \item \algRewAgg: Imitate one-step-reward using stationary policy by \Dagger (Alg.~\ref{alg:Agg})
    \item \algQvalAgg: Imitate reward-to-go using non-stationary policy by \aggrevate (Alg.~\ref{alg:Agg})
\end{enumerate}
Table.~\ref{tab:alg:mapping} shows the algorithm mapping.

 \begin{table}[!htbp]
    \centering
    \caption{Mapping from Problem and Policy type to Algorithm}
    \begin{tabulary}{0.8\textwidth}{L|CC}\toprule
       \diagbox{\bf Policy}{\bf Problem}       &   \hiddenunc         & \hiddencon        \\ \midrule
       Non-stationary policy             &    \algRewFT      & \algQvalFT           \\
       Stationary policy                     &    \algRewAgg     & \algQvalAgg          \\ \bottomrule
    \end{tabulary}
    \label{tab:alg:mapping}
\end{table}

\begin{algorithm}
\caption{Non-stationary policy (\algRewFT, \algQvalFT) \label{alg:FT}}
\begin{algorithmic}[1]
\For{$t=1$ \textbf{to} $T$} \label{alg:FT:init}
\State Initialize $\dataset \gets \emptyset$.
\For{$j=1$ \textbf{to} $\numDatapoints$}
\State Sample world map $\world$ from dataset $P(\world)$
\State Execute policy $\policyLEARN^1, \ldots, \policyLEARN^{t-1}$ to reach $\pair{\state_t}{\belief_t}$.\label{alg:FT:rollin} 
\State Execute any action $\action_t \in \actionSetFeas{\state_t}{\world}$.
\State Execute oracle $\policyOR$ from $t+1$ to $T$ on $\world$ \label{alg:FT:oracle}
\State Collect value to go $\QVal_i^{\policyOR} =  {\QFn{\policyOR}{T-t+1}(\state_t, \world, \action_t)}$
\State $\dataset \gets \dataset \cup \{\state_t, \belief_t, \action_t, t, \QVal_i^{\policyOR}\}$ 
\EndFor
\State Train cost-sensitive classifier $\policyLEARN^t$ on $\dataset$
\EndFor
\State \textbf{Return} Set of policies for each time step $\policyLEARN^1, \ldots, \policyLEARN^{T}$ .
\end{algorithmic}
\end{algorithm}

\begin{algorithm}
\caption{Stationary policy (\algRewAgg, \algQvalAgg) \label{alg:Agg}}
\begin{algorithmic}[1]
\State Initialize $\dataset \gets \emptyset$, $\policyLEARN_1$ to any policy in $\policySetBel$ \label{alg:qvalAgg:init}
\For{$i=1$ \textbf{to} $\numLearnIter$}
\State Initialize sub dataset $\dataset_i \gets \emptyset$\; \label{alg:qvalAgg:initSub}
\State Let roll-in policy be $\policyMix = \mixfrac{i} \policyOR + (1-\mixfrac{i}) \policyLEARN_i$ \label{alg:qvalAgg:mixPol}
\State Collect $m$ data points as follows:
\For{$j=1$ \textbf{to} $\numDatapoints$}
\State Sample world map $\world$ from dataset $P(\world)$ \label{alg:qvalAgg:sampleWorld}
\State Sample uniformly $t \in \{1,2,\dots,T\}$ \label{alg:qvalAgg:sampleTime}
\State Assign initial state $\state_1 = \vertexStart$ \label{alg:qvalAgg:initialState}
\State Execute $\policyMix$ up to time $t-1$ to reach $\pair{\state_t}{\belief_t}$ \label{alg:qvalAgg:rollin}
\State Execute any action $\action_t \in \actionSetFeas{\state_t}{\world}$ \label{alg:qvalAgg:takeAction}
\State Execute oracle $\policyOR$ from $t+1$ to $T$ on $\world$ \label{alg:qvalAgg:execOracle}
\State Collect value-to-go $\QVal_i^{\policyOR} =  {\QFn{\policyOR}{T-t+1}(\state_t, \world, \action_t)}$ \label{alg:qvalAgg:collectVal}
\State $\dataset_i \gets \dataset_i \cup \{\state_t, \belief_t, \action_t, t, \QVal_i^{\policyOR}\}$ \label{alg:qvalAgg:aggrSubData}
\EndFor
\State Aggregate datasets: $\dataset \gets \dataset \bigcup \dataset_i$ \label{alg:qvalAgg:aggrData}
\State Train cost-sensitive classifier $\policyLEARN_{i+1}$ on $\dataset$\\ \label{alg:qvalAgg:updateLearner}
~~~~ (\emph{Alternately: use any online learner $\policyLEARN_{i+1}$ on $\dataset_i$})
\EndFor
\State \textbf{Return} best $\policyLEARN_i$ on validation
\end{algorithmic}
\end{algorithm}

Alg.~\ref{alg:FT} describes the \FT procedure to train the non-stationary policy. Previous policies $\policyLEARN^1, \ldots, \policyLEARN^{t-1}$ are used to create a dataset (Lines~\ref{alg:FT:init}--\ref{alg:FT:rollin}). The oracle provides the value to imitate (Line~\ref{alg:FT:oracle}) - if the one-step-reward is used the algorithm is referred to \algRewFT else if an oracle is used \algQvalFT. 

Alg.~\ref{alg:Agg} describes the \aggrevate procedure to train the stationary policy. The algorithm iteratively trains a sequence of learnt policies $\seq{\policyLEARN}{\numLearnIter}$ by aggregating data for an online cost-sensitive classification problem. At any given iteration, data is collected by roll-in with a mixture of learnt and oracle policy (Lines~\ref{alg:qvalAgg:init}--\ref{alg:qvalAgg:rollin}). The oracle provides the value to imitate (Lines~\ref{alg:qvalAgg:execOracle}--\ref{alg:qvalAgg:collectVal}) - if the one-step-reward is used the algorithm is referred to \algRewAgg else if an oracle is used \algQvalAgg. Data is appended to the original dataset and used to train an updated learner $\policyLEARN_{i+1}$ (Lines~\ref{alg:qvalAgg:aggrData}--\ref{alg:qvalAgg:updateLearner}).

\subsection{Imitation of Clairvoyant One-Step-Reward}
\label{sec:approach:one_step_reward}

The strategy to imitate the clairvoyant one-step-reward is employed in Problem \hiddenunc. This is motivated by the observation that the utility function satisfies a property of adaptive submodularity (as mentioned in Section~\ref{sec:prob:hidden}). For such problems, greedily selecting actions with highest expected reward has near-optimality guarantees. This implies the following Lemma

\begin{lemma}
The performance of the hallucinating oracle $\policyORBel$ is near-optimal w.r.t the optimal policy $\policy^*$. 
\begin{equation*}
\valuePol{\policyORBel} \geq \left(1 - \frac{1}{e}\right) \valuePol{\policy^*}
\end{equation*}
\end{lemma}

This property can then be used to obtain a near-optimality bound for the learnt policy

\begin{theorem}
N iterations of \algRewAgg, collecting $m$ regression examples per iteration guarantees that with probability at least $1-\delta$
\begin{equation*}
\begin{aligned}
  \valuePol{\policyLEARN} \geq & \left(1 - \frac{1}{e}\right) \valuePol{\policy^*} \\
  & - 2\sqrt{\abs{\actionSet}}T\sqrt{\errclass + \errreg 
  + O \left(\sqrt{\log \left(\nicefrac{ \left( \nicefrac{1}{\delta} \right)}{Nm}\right)} \right)} \\
  & - O \left(\frac{T^2 \log T}{\alpha N}\right)
\end{aligned}
\end{equation*}
where $\errreg$ is the empirical average online learning regret on the training regression examples collected over the iterations and 
$\errclass$ is the empirical regression regret of the best regressor in the policy class.
\end{theorem}

\subsection{Imitation of Clairvoyant Reward-To-Go}

Problem \hiddencon does not posses the adaptive submodularity property. However there is empirical evidence to suggest that the hallucinating oracle performance $\valuePol{\policyORBel}$ is sufficiently high \cite{Koval-RSS-14}. Hence the learnt policy has a performance guarantee with respect to the hallucinating oracle

\begin{theorem}
N iterations of \algQvalAgg, collecting $m$ regression examples per iteration guarantees that with probability at least $1-\delta$
\begin{equation*}
\begin{aligned}
  \valuePol{\policyLEARN} \geq & \valuePol{\policyORBel} \\
  & - 2\sqrt{\abs{\actionSet}}T\sqrt{\errclass + \errreg 
  + O \left(\sqrt{\log \left(\nicefrac{ \left( \nicefrac{1}{\delta} \right)}{Nm}\right)} \right)} \\
  & - O \left(\frac{T^2 \log T}{\alpha N}\right)
\end{aligned}
\end{equation*}
where $\errreg$ is the empirical average online learning regret on the training regression examples collected over the iterations and 
$\errclass$ is the empirical regression regret of the best regressor in the policy class.
\end{theorem}


\section{Experimental Results}
\label{sec:res}

\begin{figure*}[!htbp]
    \centering
    \includegraphics[page=1,width=\textwidth]{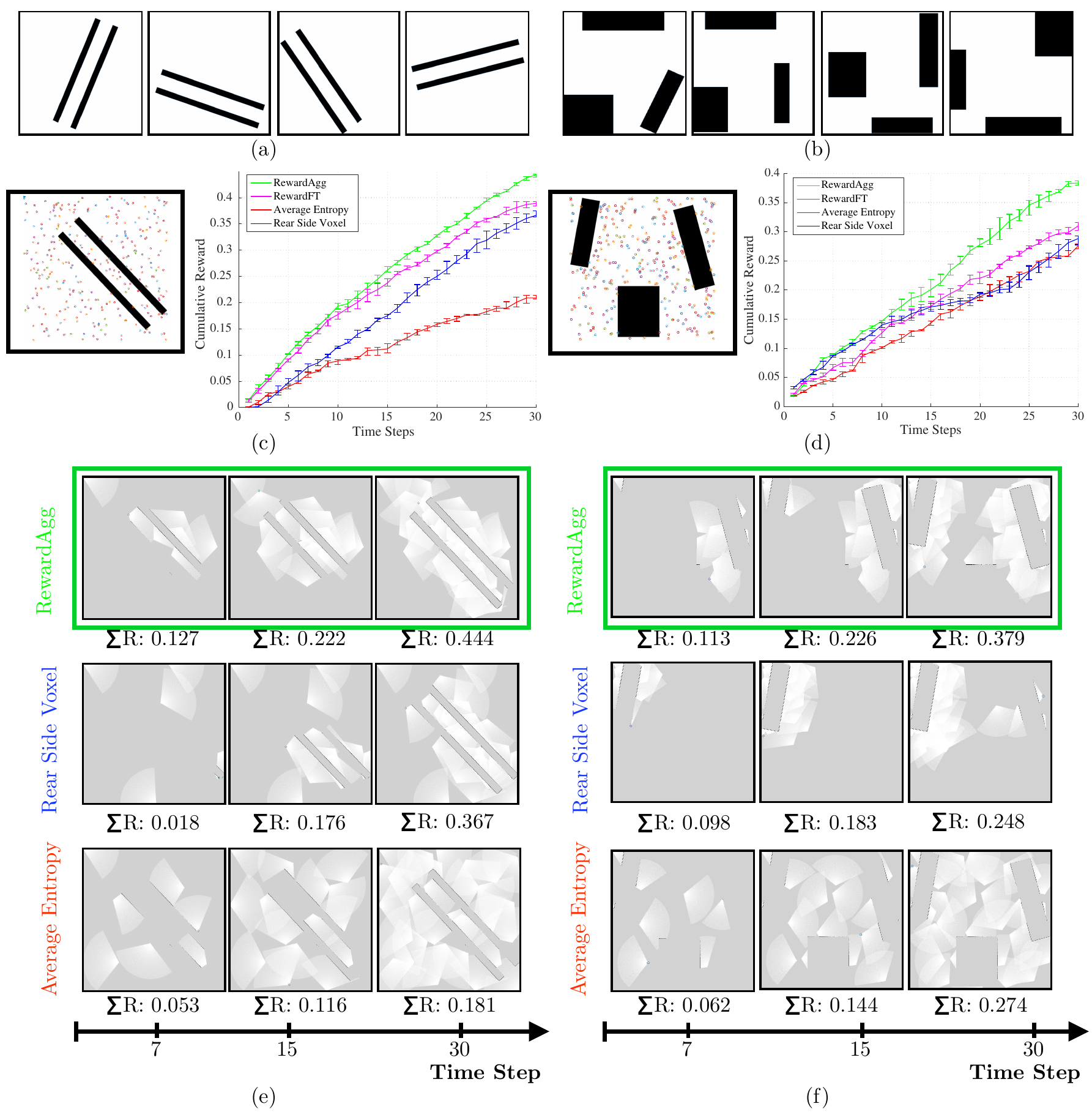}
\begin{minipage}{\textwidth}
    \caption{%
    Case study of Problem \hiddenunc using \algRewAgg, \algRewFT and baseline heuristics. Two different datasets of 2D exploration are considered - (a) dataset 1 (parallel lines) and (b) dataset 2 (distributed blocks). Problem details are: $T=30, |\actionSet|=300$, $100$ train and $100$ test maps. A sample test instance is shown along with a plot of cumulative reward with time steps for different policies is shown in (c) and (d). The error bars show $95\%$ confidence intervals. (e) and (f) show snapshots of the execution at time steps $7, 15$ and $30$. 
        \label{fig:results:matlab}
        \fullFigGap}
\end{minipage}

            \begin{minipage}{0.01\textwidth}
\phantomsubcaption{\label{fig:results:matlab:a}}
\end{minipage}
\begin{minipage}{0.01\textwidth}
\phantomsubcaption{\label{fig:results:matlab:b}}
\end{minipage}
    \begin{minipage}{0.01\textwidth}
\phantomsubcaption{\label{fig:results:matlab:c}}
\end{minipage}
\begin{minipage}{0.01\textwidth}
\phantomsubcaption{\label{fig:results:matlab:d}}
\end{minipage}
\begin{minipage}{0.01\textwidth}
\phantomsubcaption{\label{fig:results:matlab:e}}
\end{minipage}
\begin{minipage}{0.01\textwidth}
\phantomsubcaption{\label{fig:results:matlab:f}}
\end{minipage}
\end{figure*}%

\begin{figure*}[t]
    \centering
    \includegraphics[page=1,width=\textwidth]{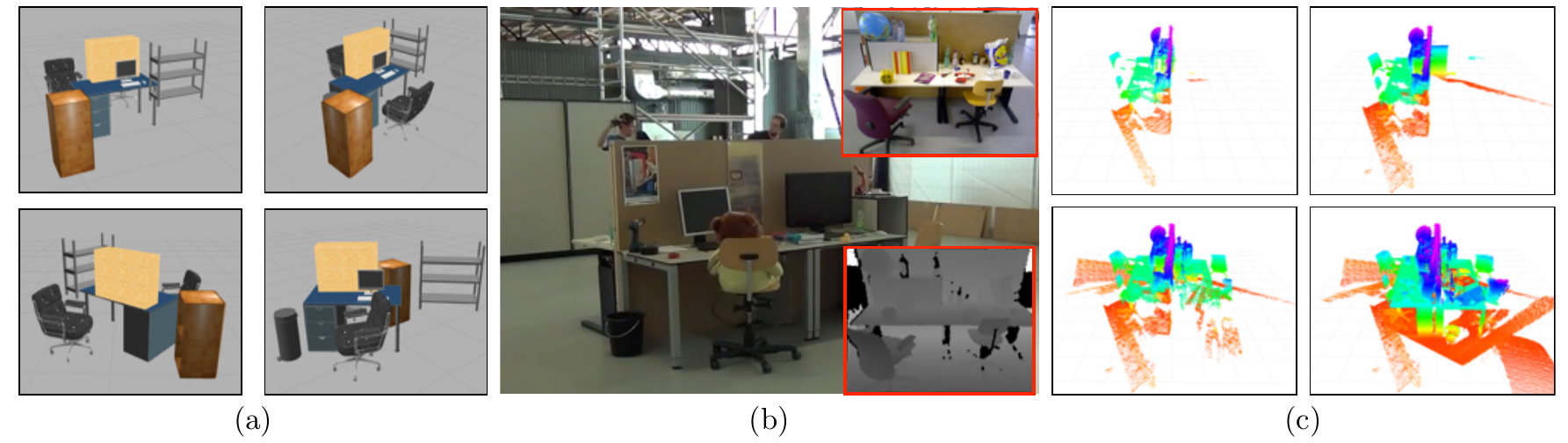}

\begin{minipage}{\textwidth}
    \caption{%
    Comparison of \algQvalAgg with baseline heuristics on a 3D exploration problem where training is done on simulated world maps and testing is done on a real dataset of an office workspace. The problem details are: $T=10$, $\costBudget=12$, $|\actionSet|=50$. 
    (a) Samples from $100$ simulated worlds resembling an office workspace created in Gazebo.
    (b) Real dataset collected by \cite{sturm12iros} using a RGBD camera. 
    (c) Snapshots of execution of \algQvalAgg heuristic at time steps $1,3, 5, 9$. 
    \label{fig:results:cpp}
        \fullFigGap}
\end{minipage}

                \begin{minipage}{0.01\textwidth}
\phantomsubcaption{\label{fig:results:cpp:a}}
\end{minipage}
\begin{minipage}{0.01\textwidth}
\phantomsubcaption{\label{fig:results:cpp:b}}
\end{minipage}
    \begin{minipage}{0.01\textwidth}
\phantomsubcaption{\label{fig:results:cpp:c}}
\end{minipage}
\end{figure*}%

\begin{figure*}[t]
    \centering
    \includegraphics[page=1,width=\textwidth]{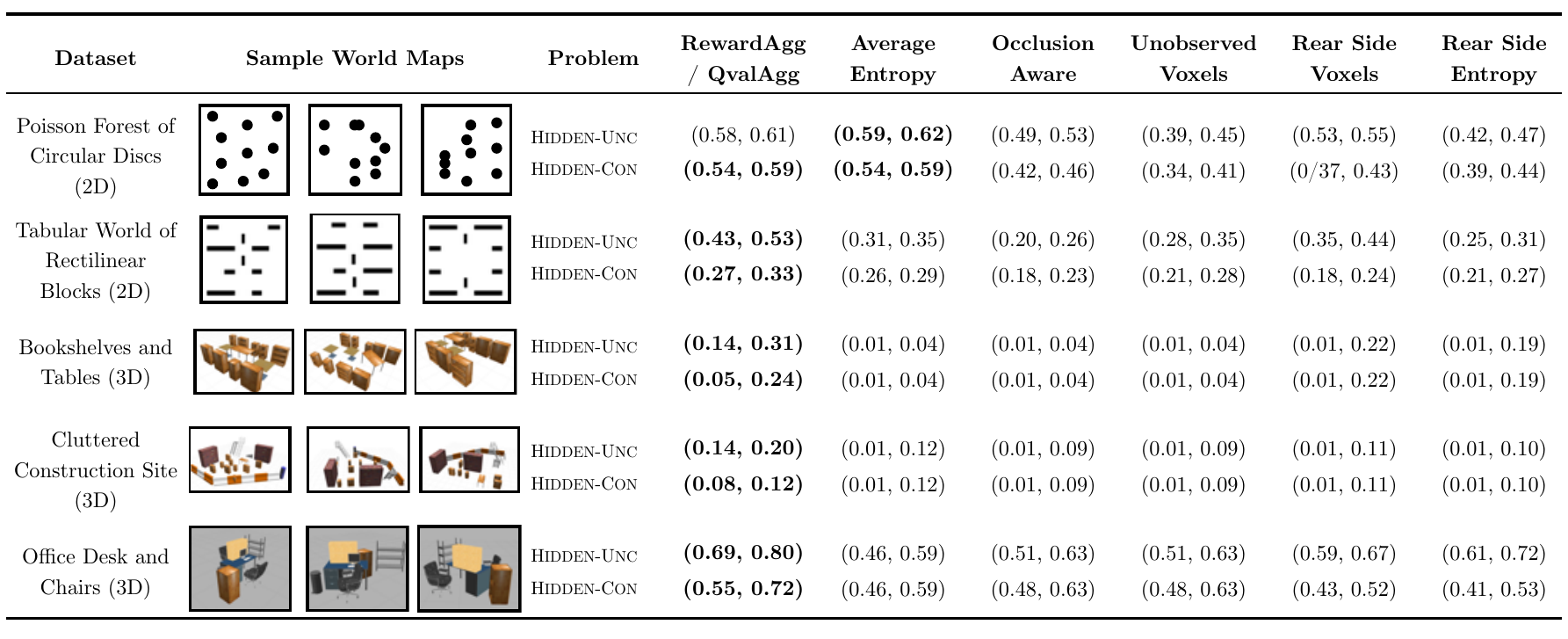}
    \caption{%
    Results for Problems \hiddenunc and \hiddencon on a spectrum of 2D and 3D exploration problems. The train size is $100$ and test size is $10$. Numbers are the confidence bounds (for 95\% CI) of cumulative reward at the final time step. Algorithm with the highest median performance is emphasized in bold. 
          \fullFigGap}
    \label{fig:results:extra}
\end{figure*}%

\subsection{Implementation Details}

\subsubsection{Problem Details} 
The utility function $\utilityFnDef$ is selected to be a fractional coverage function (similar to \cite{isler2016information}) which is defined as follows.
The world map $\world$ is represented as a voxel grid representing the surface of a 3D model. The sensor measurement $\measFn{\vertex}{\world}$ at node $\vertex$ is obtained by ray-casting on this 3D model. A voxel of the model is said to be `covered' by a measurement received at a node if a point lies in that voxel. The coverage of a path $\Path$ is the fraction of covered voxels by the union of measurements received when visiting each node of the path. The travel cost function $\costFnDef$ is chosen to be the euclidean distance.
The values of total time step $T$ and travel budget $\costBudget$ vary with problem instances and are specified along with the results.

\subsubsection{Learning Details}

The tuple $\left(\state, \action, \belief \right)$ is mapped to a vector of features $\feature =  \bbm \featureIG^T & \featureMot^T \ebm^T$. The feature vector $\featureIG$ is a vector of information gain metrics as described in \cite{isler2016information}. $\featureMot$ encodes the relative rotation and translation required to visit a node. Random forest regression is used as a function approximator. The oracle used is the generalized cost benefit algorithm (GCB) \cite{zhang2016submodular}. 

\subsubsection{Baseline}

The baseline policies are a class of information gain heuristics discussed in \cite{isler2016information} augmented with a motion penalization term when applied to Problem \hiddencon.  The heuristics are remarkably effective, however, their performance depends on the distribution of objections in a world map. 

\subsection{Adaptation to Different Distributions}
We created a set of 2D exploration problems to gain a better understanding of the learnt policies and baseline heuristics. The problem was \hiddenunc, the dataset comprises of 2D binary world maps, uniformly distributed nodes and a simulated laser. The problem details are $T=30$ and $|\actionSet|=300$. The train size is $100$, test size is $100$. \algRewAgg is executed for $10$ iterations.

The overall conclusion is that on changing the datasets the performance of the heuristics vary widely while the learnt policies outperform both heuristics. Interstingly, \algRewAgg outperforms \algRewFT - this is probably due to the generalization across time-steps. We now analyze each dataset.

\subsubsection{Dataset 1: Parallel Lines}

Fig.~\ref{fig:results:matlab:a} shows a dataset created by applying random affine transformations to a pair of parallel lines. 
This dataset is representative of information being concentrated in a particular fashion.
Fig.~\ref{fig:results:matlab:c} shows a comparison of \algRewAgg, \algRewFT with baseline heuristics. While \RearSideVoxel outperforms \AverageEntropy, \algRewAgg outperforms both.
Fig.~\ref{fig:results:matlab:e} shows progress of each. \AverageEntropy explores the whole world without focusing, \RearSideVoxel exploits early while \algQvalAgg trades off exploration and exploitation.

\subsubsection{Dataset 2: Distributed Blocks}

Fig.~\ref{fig:results:matlab:b} shows a dataset created by randomly distributing rectangular blocks around the periphery of the map.
This dataset is representative of information being distributed around.
Fig.~\ref{fig:results:matlab:c} shows that \RearSideVoxel saturates early, \AverageEntropy eventually overtaking it while \algRewAgg outperforms all.
Fig.~\ref{fig:results:matlab:e} shows that \RearSideVoxel gets stuck exploiting an island of information. \AverageEntropy takes broader sweeps of the area thus gaining more information about the world. \algQvalAgg shows a non-trivial behavior exploiting one island before moving to another.

\subsection{Train on Synthetic, Test on Real}
To show the practical usage of our pipeline, we show a scenario where a policy is trained on synthetic data and tested on a real dataset. 
Fig.~\ref{fig:results:cpp:a} shows some sample worlds created in Gazebo to represent an office desk environment on which \algQvalAgg is trained. 
Fig.~\ref{fig:results:cpp:b} shows a dataset of an office desk collected by TUM Computer Vision Group \cite{sturm12iros}. The dataset is parsed to create a pair of pose and registered point cloud which can then be used to evaluate different algorithms.
Fig.~\ref{fig:results:cpp:c} shows \algQvalAgg learns a desk exploring policy by circumnavigating around the desk. This shows the powerful generalization capabilities of the approach.

\subsection{Spectrum of 2D / 3D exploration problems}

We evaluate the framework on a spectrum of 2D / 3D exploration problems on synthetic worlds as shown in Fig.~\ref{fig:results:extra}. For Problem \hiddenunc, \algRewAgg is employed along with baseline heuristics. For Problem \hiddencon, \algQvalAgg is employed with baseline heuristic augmented with motion penalization. The train size is $100$ and test size is $10$. We see that the learnt policies outperform all heuristics on most datasets by exploiting the distribution of objects in the world particularly in Problem \hiddenunc. This is indicative of the power of the hallucinating oracle. However, the Poisson forest datasets stand out - where given the ergodic distribution, the \AverageEntropy heuristic performs best. 

\section{Conclusion}
\label{sec:conc}
We present a novel framework for learning information gathering policies via imitation learning of clairvoyant oracles. We presented analysis that establishes an equivalence to online imitation of hallucinating oracles thereby explaining the success of such policies. The framework is validated on a spectrum of 2D and 3D exploration problems. 

There are several key directions for future research. Firstly, analysis from \citet{chen2016pomdp}, where the authors show that under mild assumptions POMDPs can be reduced to sequence of MDPs, is promising for obtaining better regret guarantees. \cite{NIPS2015_6005} also offer alternatives to adaptive submodularity that could improve guarantees for \hiddencon. Secondly, instead of learning a policy, learning a surrogate utility function that is solved online with a routing TSP might lead to improved performance. Finally, we are looking to apply our framework to POMDP problems where hindsight optimization have shown success (e.g learning to grasp \citet{Koval-RSS-14}).


\clearpage

\onecolumn

\appendix

\section{Analysis using a Hallucinating Oracle}

We begin with defining a clairvoyant oracle 
\begin{definition}[Clairvoyant Oracle]
Given a distribution of world map $P(\world)$, a clairvoyant oracle $\policyOR(\state, \world)$ is a policy that maps state $\state$ and world map $\world$ to a feasible action $\action \in \actionSetFeas{\state}{\world}$ such that it approximates the optimal MDP policy, $\policyOR \approx \policyMDP = \argmaxprob{\policy \in \policySet}\valuePol{\policy}$.
\end{definition}

Our approach is to imitate the oracle during training. This implies that we train a policy $\policyLEARN$ by solving the following optimization problem

\begin{equation}
\label{eq:imitateClairvoyantOracle}
\policyLEARN = \argmax\limits_{\policyBel \in \policySetBel} \expect{
\substack{t\sim U(1:T), \\
\state \sim P(\state \mid \policyBel, t), \\
\world \sim P(\world), \\
\belief \sim P(\belief \mid \world, \policyBel, t)}}
{\QFn{\policyOR}{T-t+1}(\state, \world, \policyBel(\state,\belief))}
\end{equation}

The learnt policy imitates a clairvoyant oracle that has access to more information (world map $\world$ compared to belief $\belief$). Hence, the realizability error of the policy is due to two terms - firstly the information mismatch and secondly the expressiveness of feature space. This realizability error can be hard to bound making it difficult to bound the performance of the learnt policy. This motivates us to introduce a hypothetical construct, a \emph{hallucinating oracle}, to alleviate the information mismatch.

\begin{definition}[Hallucinating Oracle]
Given a prior distribution of world map $P(\world)$ and rollouts by a policy $\policyBel$, a hallucinating oracle $\policyORBel$ computes the instantaneous posterior distribution over world maps and takes the action with the highest expected value. 
\begin{equation}
  \policyORBel = \argmax\limits_{\action \in \actionSet} \expect{\worldTwo \sim P(\worldTwo \mid \belief, \policyBel, t)}{\QFn{\policyOR}{T-t+1}(\state, \worldTwo, \action)}
\end{equation}
\end{definition}

We will now show that by imitating a clairvoyant oracle we effectively imitate the corresponding hallucinating oracle 

\begin{lemma}
The \textbf{offline} imitation of \textbf{clairvoyant} oracle (\ref{eq:imitateClairvoyantOracle}) is equivalent to sampling \textbf{online} a world from the posterior distribution and executing a \textbf{hallucinating} oracle as shown 

\begin{equation*}
\policyLEARN = \argmax\limits_{\policyBel \in \policySetBel} \expect{
\substack{t\sim U(1:T), \\
\state \sim P(\state \mid \policyBel, t), \\
\world \sim P(\world), \\
\belief \sim P(\belief \mid \world, \policyBel, t)}}
{\QFn{\policyORBel}{T-t+1}(\state, \world, \policyBel(\state,\belief))}
\end{equation*}

by using the fact that 
$\expect
{\world \sim P(\world)}
{\QFn{\policyOR}{T-t+1}(\state, \world, \action)} = 
\expect{
\substack{\world \sim P(\world),\\
\belief \sim P(\belief \mid \world, \policyBel, t)}}
{\QFn{\policyORBel}{T-t+1}(\state, \world, \action)}
$
\end{lemma}

\begin{proof}
Given a belief $\belief$ and rollout policy $\policyBel$, the action value function for the hallucinating oracle $\QFn{\policyORBel}{T-t+1}(\state, \world, \action)$ is defined as 
\begin{equation}
	\label{eq:hallucinatingQ}
	\QFn{\policyORBel}{T-t+1}(\state, \world, \action) = \expect{\worldTwo \sim P(\worldTwo \mid \belief, \policyBel, t)}{\QFn{\policyOR}{T-t+1}(\state, \worldTwo, \action)}
\end{equation}

We will now show $\expect
{\world \sim P(\world)}
{\QFn{\policyOR}{T-t+1}(\state, \world, \action)} = 
\expect{
\substack{\world \sim P(\world),\\
\belief \sim P(\belief \mid \world, \policyBel, t)}}
{\QFn{\policyORBel}{T-t+1}(\state, \world, \action)}
$. 

Taking expectation $\expect{
\substack{\world \sim P(\world),\\
\belief \sim P(\belief \mid \world, \policyBel, t)}}
{.}$ on both sides of  (\ref{eq:hallucinatingQ}) we have
\begin{equation}
	\label{eq:middlestep}
		\expect{\substack{\world \sim P(\world),\\\belief \sim P(\belief \mid \world, \policyBel, t)}}{\QFn{\policyORBel}{T-t+1}(\state, \world, \action)}  = 
	\expect{\substack{\world \sim P(\world),\\\belief \sim P(\belief \mid \world, \policyBel, t)}}{\expect{\worldTwo \sim P(\worldTwo \mid \belief, \policyBel, t)}{\QFn{\policyOR}{T-t+1}(\state, \worldTwo, \action)}} 
\end{equation}

Evaluating the right hand side of (\ref{eq:middlestep})

\begin{equation*}
	\begin{aligned}
	& 
	\expect{\substack{\world \sim P(\world),\\\belief \sim P(\belief \mid \world, \policyBel, t)}}{\expect{\worldTwo \sim P(\worldTwo \mid \belief, \policyBel, t)}{\QFn{\policyOR}{T-t+1}(\state, \worldTwo, \action)}} &  \\
	&= \sum\limits_{\world} \sum\limits_{\belief} \sum\limits_{\worldTwo} P(\world)  P(\belief \mid \world, \policyBel, t)  P(\worldTwo \mid \belief, \policyBel, t) \QFn{\policyOR}{T-t+1}(\state, \worldTwo, \action) & \\
	&= \sum\limits_{\world} \sum\limits_{\belief} \sum\limits_{\worldTwo}  P(\world) 
	\frac{P(\world \mid \belief, \policyBel, t) P(\belief \mid \policyBel, t)}{P(\world \mid \policyBel, t)} 
	\frac{P(\belief \mid \worldTwo, \policyBel, t) P(\worldTwo \mid \policyBel, t)}{P(\belief \mid \policyBel, t)} 
	\QFn{\policyOR}{T-t+1}(\state, \worldTwo, \action) & \text{(Bayes rule)} \\
	&= \sum\limits_{\world} \sum\limits_{\belief} \sum\limits_{\worldTwo}  
	P(\world \mid \belief, \policyBel, t) 
	P(\belief \mid \worldTwo, \policyBel, t) P(\worldTwo )
	\QFn{\policyOR}{T-t+1}(\state, \worldTwo, \action) & \text{(Cond. Indep.)} \\
	&= \sum\limits_{\world} P(\worldTwo) \QFn{\policyOR}{T-t+1}(\state, \worldTwo, \action) 
	\sum\limits_{\belief} \sum\limits_{\worldTwo}  
	P(\world \mid \belief, \policyBel, t) 
	P(\belief \mid \worldTwo, \policyBel, t) P(\worldTwo ) \\
	&= \sum\limits_{\world} P(\worldTwo) \QFn{\policyOR}{T-t+1}(\state, \worldTwo, \action)  & (\text{Marginalizing } \world, \belief) \\
	&=\expect
{\world \sim P(\world)}
{\QFn{\policyOR}{T-t+1}(\state, \world, \action)} \\
	\end{aligned}
\end{equation*}

This implies that 
\begin{equation*}
\expect{
\substack{t\sim U(1:T), \\
\state \sim P(\state \mid \policyBel, t), \\
\world \sim P(\world), \\
\belief \sim P(\belief \mid \world, \policyBel, t)}}
{\QFn{\policyORBel}{T-t+1}(\state, \world, \policyBel(\state,\belief))}
=
\expect{
\substack{t\sim U(1:T), \\
\state \sim P(\state \mid \policyBel, t), \\
\world \sim P(\world), \\
\belief \sim P(\belief \mid \world, \policyBel, t)}}
{\QFn{\policyOR}{T-t+1}(\state, \world, \policyBel(\state,\belief))}
\end{equation*}

\end{proof}

\section{Mapping to \aggrevate analysis}

Now that we have established that we are imitating a reference policy that is not clairvoyant (hallucinating oracle), we can invoke the analysis used by Ross \etal \cite{ross2014reinforcement} for the algorithm \aggrevate. A key difference is that \aggrevate defines the problem as minimizing cost to go, while we define it as maximizing reward to go. Instead of re-deriving the analysis, we will instead define a term by term mapping between our framework and \aggrevate and then use the proven bounds. 

\begin{table*}[!htbp]
\centering
\caption{Mapping between \algQvalAgg and \aggrevate \label{tab:mapping}}
\begin{tabulary}{\textwidth}{LCC}\toprule
    {\bf Term}   & {\bf \algQvalAgg}         & {\bf \aggrevate}   \\ \midrule
    Objective  
    & 
    Reward: $\rewardFn{\state}{\world}{\action}$  
    & 
    Cost: $\rewardFnAgg{\state}{\world}{\action}$
    \\ 
    &
    &
    $\rewardFnAgg{\state}{\world}{\action} = 1 - \rewardFn{\state}{\world}{\action}$
    \\ \midrule
    \multicolumn{1}{m{2cm}}{State Value Function} 
    &
     $\valueFn{\policyBel}{T}(\state, \world) = \sum\limits_{i=1}^{T} \expect{\state_i \sim P(\state' \mid \state, \policyBel, i) }{\rewardFn{\state_i}{\world}{\policyBel(\state_i,\world)}}$ 
     & 
	$\valueFnAgg{\policyBel}{T}(\state, \world) = \sum\limits_{i=1}^{T} \expect{\state_i \sim P(\state' \mid \state, \policyBel, i) }{\rewardFnAgg{\state_i}{\world}{\policyBel(\state_i,\world)}}$ 
	\\
	&
	&
	$\valueFnAgg{\policyBel}{T}(\state, \world) = T - \valueFn{\policyBel}{T}(\state, \world)$
	\\ \midrule
    \multicolumn{1}{m{2cm}}{Action Value Function} 
    &
    $\QFn{\policyBel}{T}(\state, \world, \action)$
    &
	$\QFnAgg{\policyBel}{T}(\state, \world, \action) = T - \QFn{\policyBel}{T}(\state, \world, \action)$
	\\ \midrule
    \multicolumn{1}{m{2cm}}{Policy Value}  
    & 
    $\valuePol{\policyBel} = \expect{\state \sim P(\state), \world \sim P(\world)}{\valueFn{\policyBel}{T}(\state, \world)}$
    & 
    $\valuePolAgg{\policyBel} = \expect{\state \sim P(\state), \world \sim P(\world)}{\valueFn{\policyBel}{T}(\state, \world)}$
    \\ 
    &
    &
    $\valuePolAgg{\policyBel} = T - \valuePol{\policyBel}$
    \\ \midrule
        \multicolumn{1}{m{2cm}}{Loss of Policy}  
    &
    $
    \begin{aligned}
    & \lossi(\policyBel) =  
    \mathbb{E}_{
	\substack{
    t\sim U(1:T), \\
	\state \sim P(\state \mid \policyBel, t), \\
	\world \sim P(\world),\\
	\belief \sim P(\belief \mid \world, \policyBel, t)
	}
    }
    \\
    & 
    \left[
	\max\limits_{\action \in \actionSet} \QFn{\policyORBel}{T-t+1}(\state, \world, \action) - \QFn{\policyORBel}{T-t+1}(\state, \world, \policy(\state,\belief))
\right]
    \end{aligned}
    $
    &
     $
    \begin{aligned}
    & \lossiAgg(\policyBel) =  
    \mathbb{E}_{
	\substack{
    t\sim U(1:T), \\
	\state \sim P(\state \mid \policyBel, t), \\
	\world \sim P(\world),\\
	\belief \sim P(\belief \mid \world, \policyBel, t)
	}
    }
    \\
    & 
    \left[
\QFnAgg{\policyORBel}{T-t+1}(\state, \world, \policyBel) - 
\min\limits_{\action \in \actionSet} \QFnAgg{\policyORBel}{T-t+1}(\state, \world, \action) 
\right]
    \end{aligned}
    $
    \\ 
    &
    &
    $\lossiAgg = \lossi$
    \\
    \midrule
    \multicolumn{1}{m{2cm}}{Classification Error of Best Policy}  
    &
    $
    \begin{aligned}
    & \errclass = \min\limits_{\policyBel \in \policySetBel} \frac{1}{N} \sum\limits_{i=1}^N 
    \mathbb{E}_{
	\substack{
    t\sim U(1:T), \\
	\state \sim P(\state \mid \policyBel, t), \\
	\world \sim P(\world),\\
	\belief \sim P(\belief \mid \world, \policyBel, t)
	}
    }
    \\
    & 
    \left[
	\max\limits_{\action \in \actionSet} \QFn{\policyORBel}{T-t+1}(\state, \world, \action) - 
\QFn{\policyORBel}{T-t+1}(\state, \world, \policyBel)
\right]
    \end{aligned}
    $
    &
    $
    \begin{aligned}
    & \errclassAgg = \min\limits_{\policyBel \in \policySetBel} \frac{1}{N} \sum\limits_{i=1}^N 
    \mathbb{E}_{
	\substack{
    t\sim U(1:T), \\
	\state \sim P(\state \mid \policyBel, t), \\
	\world \sim P(\world),\\
	\belief \sim P(\belief \mid \world, \policyBel, t)
	}
    }
    \\
    & 
    \left[
\QFnAgg{\policyORBel}{T-t+1}(\state, \world, \policyBel) - 
\min\limits_{\action \in \actionSet} \QFnAgg{\policyORBel}{T-t+1}(\state, \world, \action) 
\right]
    \end{aligned}
    $
    \\
    &
    &
    $\errclassAgg = \errclass$
    \\ \midrule
    \multicolumn{1}{m{2cm}}{Regret}  
    &
    $\errreg = \frac{1}{N} \left[ \sum\limits_{i=1}^N \lossi (\hat{\policy}_i) - \min\limits_{\policy \in \policySetBel} \sum\limits_{i=1}^N \lossi (\policy) \right]
    $
    &
    $\errregAgg = \frac{1}{N} \left[ \sum\limits_{i=1}^N \lossiAgg (\hat{\policy}_i) - \min\limits_{\policy \in \policySetBel} \sum\limits_{i=1}^N \lossiAgg (\policy) \right]
    $
    \\
   &
    &
    $\errregAgg = \errreg$
    \\
    \bottomrule
\end{tabulary}
\label{tab:learning_details}
\end{table*}

The \aggrevate proofs also defines a bound $\overbar{Q}_\mathrm{max}^*$ which is the performance bound of reference policy. This is mapped as follows
\begin{equation}
	\label{eq:constant}
	\begin{aligned}
	\overbar{Q}_\mathrm{max}^* &= \max\limits_{(\state, \action, t)} t - \QFn{\policyORBel}{t}(\state, \world, \action) \\
							   &\leq T
	\end{aligned}
\end{equation}

\section{Performance Bounds from \aggrevate}

\begin{lemma}
N iterations of \algQvalAgg for the infinite sample case

\begin{equation*}
	\valuePol{\policyLEARN} \geq \valuePol{\policyORBel} - T[\errclass + \errreg] - O \left(\frac{T^2 \log T}{\alpha N}\right)
\end{equation*}
\end{lemma}

\begin{proof}
Theorem 2.1 in \cite{ross2014reinforcement} states
\begin{equation*}
	\valuePolAgg{\policyLEARN} \leq \valuePolAgg{\policyORBel} + T[\errclassAgg + \errregAgg] + O \left(\frac{\overbar{Q}_\mathrm{max}^* T \log T}{\alpha N}\right)
\end{equation*}

Substituting terms from Table.~\ref{tab:learning_details} and (\ref{eq:constant}) we get

\begin{equation*}
	\begin{aligned}
	\valuePolAgg{\policyLEARN} & \leq \valuePolAgg{\policyORBel} + T[\errclassAgg + \errregAgg] + O \left(\frac{\overbar{Q}_\mathrm{max}^* T \log T}{\alpha N}\right) \\
	T - \valuePol{\policyLEARN} & \leq T - \valuePol{\policyORBel} + T[\errclassAgg + \errregAgg] + O \left(\frac{T^2 \log T}{\alpha N}\right) \\
	\valuePol{\policyLEARN} & \geq \valuePol{\policyORBel} - T[\errclass + \errreg] - O \left(\frac{T^2 \log T}{\alpha N}\right)
	\end{aligned}
\end{equation*}

\end{proof}

\begin{lemma}
\label{lemma:main}
N iterations of \algQvalAgg, collecting $m$ regression examples per iteration guarantees that with probability at least $1-\delta$
\begin{equation*}
  \valuePol{\policyLEARN} \geq \valuePol{\policyORBel} - 2\sqrt{\abs{\actionSet}}T\sqrt{\hat{\errclass} + \hat{\errreg} 
  + O \left(\sqrt{\log \left(\nicefrac{ \left( \nicefrac{1}{\delta} \right)}{Nm}\right)} \right)} - O \left(\frac{T^2 \log T}{\alpha N}\right)
\end{equation*}
where $\hat{\errreg}$ is the empirical average online learning regret on the training regression examples collected over the iterations and 
$\hat{\errclass}$ is the empirical regression regret of the best regressor in the policy class.
\end{lemma}

\begin{proof}
Theorem 2.2 in \cite{ross2014reinforcement} states
\begin{equation*}
	\valuePolAgg{\policyLEARN}  \leq \valuePolAgg{\policyORBel} + 
	2\sqrt{\abs{\actionSet}}T\sqrt{\hat{\errclassAgg} + \hat{\errregAgg} 
	+ O \left(\sqrt{\log \left(\nicefrac{ \left( \nicefrac{1}{\delta} \right)}{Nm}\right)} \right)}
	 + O \left(\frac{\overbar{Q}_\mathrm{max}^* T \log T}{\alpha N}\right)
\end{equation*}

Substituting terms from Table.~\ref{tab:learning_details} and (\ref{eq:constant}) we get

\begin{equation*}
	\begin{aligned}
	\valuePolAgg{\policyLEARN} & \leq \valuePolAgg{\policyORBel} + 
	2\sqrt{\abs{\actionSet}}T\sqrt{\hat{\errclassAgg} + \hat{\errregAgg} 
	+ O \left(\sqrt{\log \left(\nicefrac{ \left( \nicefrac{1}{\delta} \right)}{Nm}\right)} \right)}
	 + O \left(\frac{\overbar{Q}_\mathrm{max}^* T \log T}{\alpha N}\right) \\
	T - \valuePol{\policyLEARN} & \leq T - \valuePol{\policyORBel} + 
	2\sqrt{\abs{\actionSet}}T\sqrt{\errclass + \errreg 
	+ O \left(\sqrt{\log \left(\nicefrac{ \left( \nicefrac{1}{\delta} \right)}{Nm}\right)} \right)}
	 + O \left(\frac{T^2 \log T}{\alpha N}\right) \\
  \valuePol{\policyLEARN} & \geq \valuePol{\policyORBel} - 2\sqrt{\abs{\actionSet}}T\sqrt{\hat{\errclass} + \hat{\errreg} 
  + O \left(\sqrt{\log \left(\nicefrac{ \left( \nicefrac{1}{\delta} \right)}{Nm}\right)} \right)} - O \left(\frac{T^2 \log T}{\alpha N}\right)	\end{aligned}
\end{equation*}

\end{proof}

\section{Imitation of Clairvoyant One-Step-Reward}

\begin{lemma}
The performance of the hallucinating oracle $\policyORBel$ is near-optimal w.r.t the optimal policy $\policy^*$. 
\begin{equation*}
\valuePol{\policyORBel} \geq \left(1 - \frac{1}{e}\right) \valuePol{\policy^*}
\end{equation*}
\end{lemma}

\begin{proof}

Problem \hiddenunc is maximization of an adaptive submodular function subject to cardinality constraints. An apdative greedy approach to solving this problem has a near-optimality guarantee as proved by Golovin and Krause \cite{golovin2011adaptive}. We will establish an equivalence between the hallucinating one-step-reward oracle and the greedy algorithm.

The greedy algorithm selects a node to visit that has the highest expected marginal gain under the conditional distribution of world maps given the observations received so far. If the history of vertices visited and measurements received are $\{ \vertex_j \}_{j=1}^{i-1}$ and 
$\{ \meas_j \}_{j=1}^{i-1}$, the greedy algorithm selects element $\vertex_i$ with the highest expected marginal gain
\begin{equation}
\label{eq:adaptive_greedy}
\vertex_{i} = \argmaxprob{\vertex \in \vertexSet} \expect{\world \sim P(\world | \{ \vertex_j \}_{j=1}^{i-1}, \{ \meas_j \}_{j=1}^{i-1} )}{\marginalGain{\vertex | \{ \vertex_j \}_{j=1}^{i-1} }{\world}}
\end{equation}

Golovin and Krause \cite{golovin2011adaptive} prove the following: Let $\utilityFnDef$ be a utility function satisfying the properties of a non-negative adaptive montonone and adaptive submodular set function with respect to distribution $P(\world)$.
Let $\adaptivePath^*$ be the optimal solution to Problem \hiddenunc. 
Let $\adaptivePathGreedy$ be the solution obtained by applying the greedy algorithm (\ref{eq:adaptive_greedy}). 
Let $\expect{\world \sim P(\world)}{\utilityFn{\adaptivePath}{\world}}$ be the expected utility of a policy when evaluated on a distribution of world maps $P(\world)$.
Then the following guarantee holds
\begin{equation*}
  \expect{\world \sim P(\world)}{\utilityFn{\adaptivePathGreedy}{\world}} \geq \left( 1 - e^{-1}\right) \expect{\world \sim P(\world)}{\utilityFn{\adaptivePath^*}{\world}}
\end{equation*}

We now show the equivalence of hallucinating one-step-reward oracle to greedy algorithm. The hallucinating one-step-reward oracle is

\begin{equation}
  \policyORBel = \argmax\limits_{\action \in \actionSet} \expect{\worldTwo \sim P(\worldTwo \mid \belief, \policyBel, t)}{\rewardFn{\state}{\worldTwo}{\action}}
\end{equation}

This can be expanded as 

\begin{equation*}
\begin{aligned}
  \policyORBel &= \argmax\limits_{\action \in \actionSet} \expect{\worldTwo \sim P(\worldTwo \mid \belief, \policyBel, t)}{\rewardFn{\state}{\worldTwo}{\action}} \\
  &= \argmax\limits_{\action \in \actionSet} \expect{\worldTwo \sim P(\worldTwo \mid \belief, \policyBel, t)}{ \frac{\marginalGain{\action \mid \state}{\worldTwo}}{\utilityFn{\actionSet}{\worldTwo}}  } \\
  & = \argmax\limits_{\action \in \actionSet} \expect{\worldTwo \sim P(\worldTwo \mid \belief, \policyBel, t)}{ \marginalGain{\action \mid \state}{\worldTwo} }
  \end{aligned}
\end{equation*}

Since $\state = \{ \vertex_j \}_{j=1}^{t}$, $\actionSet = \vertexSet$ and $\belief_t = \{(\vertex_j, \meas_j)\}_{j=1}^{t}$

\begin{equation*}
\begin{aligned}
  \policyORBel &= \argmax\limits_{\vertex \in \vertexSet} \expect{\worldTwo \sim P(\worldTwo | \{ \vertex_j \}_{j=1}^{t}, \{ \meas_j \}_{j=1}^{t} ) }{ \marginalGain{\vertex \mid  \{ \vertex_j \}_{j=1}^{t} }{\worldTwo} }
  \end{aligned}
\end{equation*}

Hence the $\policyORBel$ is equivalent to the greedy algorithm and we get the following guarantee

\begin{equation*}
\valuePol{\policyORBel} \geq \left(1 - \frac{1}{e}\right) \valuePol{\policy^*}
\end{equation*}

\end{proof}

Now proving the regret guarantee is a straightforward conjunction of lemmas

\begin{theorem}
N iterations of \algRewAgg, collecting $m$ regression examples per iteration guarantees that with probability at least $1-\delta$
\begin{equation*}
\begin{aligned}
  \valuePol{\policyLEARN} \geq & \left(1 - \frac{1}{e}\right) \valuePol{\policy^*} \\
  & - 2\sqrt{\abs{\actionSet}}T\sqrt{\errclass + \errreg 
  + O \left(\sqrt{\log \left(\nicefrac{ \left( \nicefrac{1}{\delta} \right)}{Nm}\right)} \right)} \\
  & - O \left(\frac{T^2 \log T}{\alpha N}\right)
\end{aligned}
\end{equation*}
where $\errreg$ is the empirical average online learning regret on the training regression examples collected over the iterations and 
$\errclass$ is the empirical regression regret of the best regressor in the policy class.
\end{theorem}

\begin{proof}
In Lemma \ref{lemma:main} use $\valuePol{\policyORBel} \geq \left(1 - \frac{1}{e}\right) \valuePol{\policy^*}$.
\end{proof}

\section{Imitation of Clairvoyant Reward-To-Go}

\begin{theorem}
N iterations of \algQvalAgg, collecting $m$ regression examples per iteration guarantees that with probability at least $1-\delta$
\begin{equation*}
\begin{aligned}
  \valuePol{\policyLEARN} \geq & \valuePol{\policyORBel} \\
  & - 2\sqrt{\abs{\actionSet}}T\sqrt{\errclass + \errreg 
  + O \left(\sqrt{\log \left(\nicefrac{ \left( \nicefrac{1}{\delta} \right)}{Nm}\right)} \right)} \\
  & - O \left(\frac{T^2 \log T}{\alpha N}\right)
\end{aligned}
\end{equation*}
where $\errreg$ is the empirical average online learning regret on the training regression examples collected over the iterations and 
$\errclass$ is the empirical regression regret of the best regressor in the policy class.
\end{theorem}

\begin{proof}
This is same as Lemma \ref{lemma:main} 
\end{proof}

\clearpage

\twocolumn

\bibliographystyle{plainnat}
\bibliography{all}

\end{document}